\documentclass{article}
\usepackage{spconf,amsmath,graphicx}% DO 
\usepackage{amsmath,graphicx}% DO 
\usepackage{amsmath,amsfonts}
\usepackage{array}
\usepackage[caption=false,font=normalsize,labelfont=sf,textfont=sf]{subfig}
\usepackage{textcomp}
\usepackage{stfloats}
\usepackage{url}
\usepackage{verbatim}
\usepackage{graphicx}
\usepackage{booktabs}
\usepackage{cite}
\usepackage{amsmath}

\newcommand{\defeq}{\mbox {$  \ \stackrel{\Delta}{=} $}}

\usepackage{color}
\usepackage{url}
\definecolor{indigo}{rgb}{0.0, 0.25, 0.42}
\definecolor{darkorange}{rgb}{1.0, 0.55, 0.0}
\definecolor{darkblue}{rgb}{0.122, 0.435, 0.698}

\usepackage{microtype}
\usepackage{graphicx}
\usepackage{booktabs} % for professional tables

\usepackage[lined,boxed,commentsnumbered,ruled]{algorithm2e}
\usepackage{algorithmic}

\usepackage{amsmath}
\usepackage{amssymb}
\usepackage{mathtools}
\usepackage{amsthm}
\usepackage{stfloats}
\let\oldpar\paragraph
\renewcommand{\paragraph}[1]{\vspace{-5mm}\oldpar*{#1}}

\usepackage{amsmath}
\usepackage{amssymb}
\usepackage{mathtools}
\usepackage{amsthm}
\usepackage{stfloats}
\usepackage{hyperref}
\usepackage[capitalize,noabbrev]{cleveref}

\newtheorem{theorem}{Theorem}
\newtheorem{lemma}{Lemma}

\newcommand{\dif}{\mathrm{d}}

\begin{document}

\setlength{\abovedisplayskip}{3pt}
\setlength{\belowdisplayskip}{4pt}

\title{Rate-Constrained Quantization for Communication-Efficient Federated Learning}

% \name{Shayan~Mohajer~Hamidi$^1$, Ali Bereyhi$^2$}
% \address{$^1$University of Waterloo, Dept. of Electrical and Computer Engineering, Waterloo, Canada.\\$^2$University of Toronto, Dept. of Electrical and Computer Engineering, Toronto, Canada.}

\twoauthors{%
    Shayan Mohajer Hamidi\vspace{-3mm}
    % \thanks{}
}{%
    ECE Department, 
    University of Waterloo \\
%    Address A-B \\
   \small \texttt{smohajer@uwaterloo.ca}
}{%
   Ali Bereyhi\vspace{-3mm}
}{%
 ECE Department, 
    University of Toronto \\
%    Address A-B \\
   \small \texttt{ali.bereyhi@utoronto.ca}
}

\maketitle
\begin{abstract}
%Federated learning (FL) is a distributed learning framework where multiple devices collaboratively train a shared model with the assistance of a parameter server (PS). 
%Communication efficiencyederated learning (FL), clients transmit their local parameters to the parameter server (PS). This incurs a significant communication cost. 
Quantization is a common approach to mitigate the communication cost of federated learning (FL). In practice, the quantized local parameters are further encoded via an entropy coding technique, such as Huffman coding, for efficient data compression. In this case, the exact communication overhead is determined by the bit rate of the encoded gradients. Recognizing this fact, this work deviates from the existing approaches in the literature and develops a novel quantized FL framework, called \textbf{r}ate-\textbf{c}onstrained \textbf{fed}erated learning (RC-FED), in which the gradients are quantized subject to both fidelity and data rate constraints. We formulate this scheme, as a joint optimization in which the quantization distortion is minimized while the rate of encoded gradients is kept below a target threshold. This enables for a tunable trade-off between quantization distortion and communication cost. We analyze the convergence behavior of RC-FED, and show its superior performance against baseline quantized FL schemes on several datasets.
\end{abstract}

\begin{keywords}
Federated learning, scalar quantization, source coding, Lloyd-Max algorithm.
\end{keywords}

\section{Introduction}
\label{Sec:Introduction}
Federated learning (FL) enables distributed clients with local datasets to collaboratively train a global model through multiple iterations with a parameter server (PS) \cite{mcmahan2017communication}. FL framework involves three key steps in each iteration \cite{mcmahan2017communication}: ($i$) the PS broadcasts the current global model to all clients. ($ii$) Each device trains the model over its local dataset for some iterations and sends back the locally-updated model, and ($iii$) the PS updated the global model by the aggregation of the local models and starts over with step ($i$). As the expense of enabling distributed training, FL faces some practical challenges, such as communication overhead \cite{9357490,9833972,10487854} and data heterogeneity \cite{10619204,9743558}. The former is particularly crucial, due to unreliable network connections, limited resources, and high latency of many practical data networks, e.g., wireless systems \cite{amiri2020federated,10487854}. 

To improve communication efficiency of FL, various approaches have been proposed in the literature, e.g., gradient quantization \cite{alistarh2017qsgd}, sparsification \cite{stich2018sparsified} and model pruning \cite{jiang2022model}. Among these approaches, \textit{gradient quantization} is a promising method \cite{9305988} in which the gradients are quantized by being represented via a lower number of bits, thereby reducing the communication overhead. This work looks at gradient quantization from an information-theoretic perspective and develops a novel quantization algorithm for gradient compression.

\subsection{Related Work}
Training a model via quantized mini-batch gradients is studied in \cite{alistarh2017qsgd}, where the quantized stochastic gradient descent (QSGD) scheme is proposed. The authors in \cite{amiri2020federated} study the convergence of a lossy FL scheme, where both global and local updates are quantized before transmission. A universal vector quantization scheme for FL is proposed in \cite{9305988} whose error is shown to diminish as the number of clients increases. The authors in \cite{elgabli2020q} study gradient quantization in a decentralized setting. A heterogeneity-aware quantization scheme is proposed in  \cite{chen2021dynamic}, which improves  robustness against heterogeneous quantization errors across the network. Using nonuniform quantization for FL is proposed in \cite{chen2023nqfl}.

% In \cite{chen2021dynamic}, a heterogeneous quantization scheme was designed to minimize the convergence upper bound as a function of clients' heterogeneous quantization errors. A decentralized quantization algorithm based on the alternating direction method of multipliers was proposed in \cite{elgabli2020q}, reducing communication costs compared to non-quantized methods. Finally, studies \cite{kim2022tradeoff} and \cite{liu2022training} explored weight quantization methods in FL over wireless networks to minimize energy consumption and training time.

% \tcb{
% The proposed algorithms in the literature often focus on minimizing the distortion caused by quantization aiming to reduce the impacts of quantization on the convergence of classical training algorithms. An alternative approach is to develop efficient quantization methods while considering a quantization-aware training algorithm, such as quantized stochastic gradient descent (QSGD) \cite{alistarh2017qsgd}, 
% } which allow for smooth trade-off between communication and convergence. %
% % 

In most practical settings, quantized gradients are non-uniformly distributed over their alphabet. Quantization is thus often followed by an entropy source encoder, e.g., Huffman coding, before being transmitted to the PS \cite{shlezinger2020federated}. In such cases, the overall compression rate is described by comparing the local gradients with their quantized counterparts \textit{after entropy encoding}. This suggests that in these settings the compression rate after entropy encoding is a natural objective whose optimization can lead to efficient quantization. Though intuitive, such design has been left unaddressed in the literature. Motivated by that, in this work, we study a quantized FL framework whose objective is to minimize the rate after encoding. The most related study to this work is \cite{chen2024communication} which invokes Lloyd-Max algorithm to develop a new quantization scheme with adaptive levels. The scheme in \cite{chen2024communication} however follows the traditional distortion minimization for quantization. Unlike \cite{chen2024communication}, we design a quantization scheme that constrains the rate after encoding the quantized parameters. 

\subsection{Contributions}
Acknowledging the trade-off between quantization distortion and communication overhead, this paper proposes a novel framework for FL, dubbed \textbf{r}ate-\textbf{c}onstrained \textbf{fed}erated learning (RC-FED). Unlike existing methods that primarily prioritize distortion reduction, RC-FED aims to minimize distortion while simultaneously ensuring that the encoded gradient rate remains below a predefined threshold. This is achieved by incorporating an entropy-aided compression constraint into the traditional distortion minimization objective. By doing so, we select the quantization scheme that offers the optimal balance between distortion and compressibility, as measured by information-theoretic criteria. 
% Consequently, devices can use entropy coding methods to encode these quantized gradients, significantly reducing communication costs.
To further reduce the communication overhead, we incorporate a universal quantization technique into RC-FED, which eliminates the need for communicating the quantization hyperparameters over the network. 
% This enables  clients can first normalize their gradient vectors to a standard Gaussian distribution. This normalization ensures that the gradient l distribution is the same across clients, allowing the use of a universal quantizer for all of them. 

% In this approach, the PS designs the quantization parameters and shares them with the clients at the beginning of the FL task.

% Consequently, clients only need to share the mean and standard deviation of their local gradient distribution with the PS, enabling it to recover the original gradients. 

The key contributions of this paper are as follows: ($i$) we introduce RC-FED, a communication-efficient FL framework that directly optimizes the end-to-end compression rate of clients. ($ii$) Invoking the results of \cite{zhang2022fundamental,lee2018deep}, we approximate the local gradient distributions with their Gaussian limits whose parameters can be derived empirically. We then use this approximation to develop a \textit{universal} quantization algorithm, in which the clients do not require any hyperparameter exchange during the training phase. ($iii$)
We investigate the convergence behaviour of RC-FED and show that its convergence rate is $\mathcal{O}(\frac{1}{t})$. ($iv$) We validate RC-FED by numerical experiments on the FEMNIST and CIFAR-10 datasets.

\paragraph{Notation} 
Vectors are denoted by bold-face letters, e.g., $\boldsymbol{w}$. The $i$-th entry of $\boldsymbol{w}$ is denoted by $\boldsymbol{w}[i]$, and $\boldsymbol{w}^{\rm T}$ denotes the transpose of $\boldsymbol{w}$. Mathematical expectation is shown by $\mathbb{E} \{ \cdot \}$. For a positive integer $K$, the set $\{1,\dots,K\}$ is shortened as $[K]$. The real axis is represented by $\mathbb{R}$.

\section{Preliminaries} 
% \subsection{Federated Learning}
We consider a classical FL setting with $K$ clients, where the target learning task can be written as %FL task is defined as follows
\begin{align} \label{eq:FL}
 \min_{\boldsymbol{\theta}} \left[ f(\boldsymbol{\theta}) \triangleq \frac{1}{K} \sum_{k \in [K]} f_k \left(\boldsymbol{\theta}\right) \right].
\end{align} 
Here, $\boldsymbol{\theta} \in \mathbb{R}^d$ represents the parameters of the global model, and $f_k \left( \cdot \right)$ denotes the empirical loss computed by a local mini-batch at client $k$. The function $f\left(\cdot\right)$ further denotes the global loss being defined as the aggregation of local empirical losses.

The classical approach to problem \eqref{eq:FL} is to use distributed stochastic gradient descent (DSGD). Without loss of generality, we consider the basic case with single local iterations: in the $t$-th round of DSGD, the PS shares its latest update of the global model, i.e., $\boldsymbol{\theta}_t$, with all clients. Each client, upon receiving $\boldsymbol{\theta}_t$, computes its local gradient and transmits it to the PS which aggregates them into the global gradient, i.e., the gradient of the aggregated batch, and performs one SGD step as 
% \small
\begin{align}
\boldsymbol{\theta}_{t+1} = \boldsymbol{\theta}_t - \eta_t \nabla f(\boldsymbol{\theta}_t)=
\boldsymbol{\theta}_t - \frac{\eta_t}{K} \sum_{k \in [K]} \nabla f_k(\boldsymbol{\theta}_t),    
\end{align}
% \normalsize
where $\eta_t$ is the global learning rate. For ease of notation, we denote local gradient $\nabla f_k(\boldsymbol{\theta}_t)$  
by $\boldsymbol{\mathfrak{g}}_{k,t}$ hereafter. We assume that the DSGD algorithm runs for $T^{\max}$ iterations till it converges.

\paragraph{Universal Quantization}
As uploading throughput is typically more limited compared to its downloading counterpart, a common practice is to have the $k$-th user to communicate a finite-bit quantized representation of its gradients. The quantization operation is hence defined to be the procedure of encoding the gradients into a set of bits: for a given number of bits $b$, the quantization operation $\mathsf{Q}(\cdot):\mathbb{R}\mapsto \{s_l: {l \in [2^b]} \}$ maps the entry $i\in [d]$ of gradient $\boldsymbol{\mathfrak{g}}_{k,t}$ into its $b$-bit quantized form, i.e., $\hat{\boldsymbol{\mathfrak{g}}}_{k,t}[i] = \mathsf{Q} (\boldsymbol{\mathfrak{g}}_{k,t}[i])$. We call this quantizer \textit{universal}, as it remains unchanged for all $k\in [K]$ and $t \in [T^{\max}]$.

\paragraph{Source-encoded Transmission}
The quantized gradients are source-encoded for further compression. We assume that an \textit{entropy coding}, e.g., Huffman or Lempel-Ziv scheme, is used for compression. By \textit{entropy coding}, we refer to source coding schemes whose compression rates in the large limit converge to Shannon's bound, i.e., entropy of the quantized source.

% Specifically, assume a universal gradient quantization using $b$ bits, the quantization operation $\mathsf{Q}(\cdot)$ is a mapping from $\mathbb{R}$ into a set of quantization levels $\{s_l\}_{l \in [2^b]}$. Thus, $\boldsymbol{\mathfrak{g}}_{k,t}[i]$, $\forall i \in [d], \forall k \in [K], \forall t \in [T^{\max}]$ is quantized to $\hat{\boldsymbol{\mathfrak{g}}}_{k,t}[i]$ which is an element in $\{s_l\}_{l \in [2^b]}$. 

\section{Rate-Constrained FL} \label{sec:meth}
%The ultimate goal is to design a quantized FL scheme, in which the distortion caused by quantization is minimized while the overall compression rate of clients, i.e., bit rate after source coding, is restricted. To this end, we develop
% 
% In this section, we elaborate on the 
The RC-FED framework consists of 
four core components: ($i$) gradient normalization, ($ii$) gradient quantization, ($iii$) gradient transmission scheme, and ($iv$) gradient accumulation at the PS. These components are illustrated below.

\subsection{Gradient Normalization} \label{sec:normal}
If each client uses a \textit{personalized} quantizer, i.e., quantizer whose hyperparameters are tuned for the client, it needs to share its quantization hyperparameters with the PS. This leads to further communication overhead, which is undesired. We hence develop a \textit{universal} scheme in which the quantization hyperparameters are the same %Therefore, in RC-FED, our goal is to deploy a unified quantizer 
across all clients. In this case, as long as the bit rate remains unchanged, the quantizer hyperparameters need to be computed once at the beginning of the training phase by the PS. 

We enable universal quantization through a statistics-aware gradient normalization: %that is illustrated in the sequel. % across all clients.
as shown in \cite{zhang2022fundamental,lee2018deep}, the distribution of local gradients tend to Gaussian over the course of training.\footnote{This is shown under a set of limiting properties for the model that we assume holding approximately.} Assuming large learning model, we use this result to approximate local gradient of client $ k \in [K]$ at round $t$ as samples of Gaussian distribution $\mathcal{N} (\mu_{k,t}, \sigma^2_{k,t})$ with mean $\mu_{k,t}$ and standard deviation $\sigma_{k,t}$ computed empirically from $\boldsymbol{\mathfrak{g}}_{k,t}$.

The Gaussian approximation of model gradients enables clients to independently process their local gradients into normalized versions with the same statistics: at round $t$, client $k$ normalizes 
%$\boldsymbol{\mathfrak{g}}_{k,t}$ 
its local gradient as
$\Tilde{\boldsymbol{\mathfrak{g}}}_{k,t} = (\boldsymbol{\mathfrak{g}}_{k,t} - \mu_{k,t} )/{\sigma_{k,t}}$. These normalized versions can be approximated by $\mathcal{N}(0,1)$ for all $k\in [K]$. This enables the design of a universal quantization scheme that is discussed next. %
%Next, client $k$ quantizes $\Tilde{\boldsymbol{\mathfrak{g}}}_{k,t}$ in the manner discussed in the following subsection. 
% \begin{remark}
% For a fixed number of quantization bits, the quantization parameters remain the same across different clients. As such, the local clients have no need to send the quantizer parameters to the PS.    
% \end{remark}

\subsection{Gradient Quantization with Constrained Rate} \label{sec:quant}

%In this subsection, our objective is to design a quantization operation $\mathsf{Q}(\cdot)$ through which client $k$ quantizes the normalized gradient $\Tilde{\boldsymbol{\mathfrak{g}}}_{k,t}$, $\forall k \in [K],  \forall t \in [\text{T}^{\rm max}]$, and then sends this quantized gradient to the PS. To achieve this, 
We now treat the normalized gradient entries as random variables. Since they are distributed identically, we use the notation $Z$ to refer to a sample normalized entry. As mentioned, the distribution of $Z$ is approximated by $\mathcal{N}(0,1)$; however, for sake of generality, we consider a general distribution with probability density function (PDF) $f_Z(\cdot)$. The normalized gradients are quantized by $\mathsf{Q}(\cdot)$ with $b$ bits. In the sequel, we show the quantization levels with $s_l$ and the boundaries with $u_l$ for ${l \in [2^b]}$, i.e., $\mathsf{Q}(z) = s_l$, if $ u_{l} < z \leq u_{l+1}$. Our goal is to find efficient choice of levels and boundaries.

% To design $\mathsf{Q}(z)$, we aim to 
Traditional designs try to minimize a distortion measure between $z$ and $\mathsf{Q}(z)$, i.e., the quantization distortion. A classical measure is the mean squared error (MSE) computed as %. Specifically, the MSE between $z$ and  $\mathsf{Q}(z)$ is obtained as follows
\begin{align}
\mathsf{MSE}_{\mathsf{Q}} (Z) = \sum_{l \in [2^b]} \int_{u_l}^{u_{l+1}} (z-s_l)^2 f_Z({z}) \dif {z},
\end{align}
which we adopt in this work. %
% where $\mathsf{MSE} (\Tilde{z},s_l)=(\Tilde{z}-s_l)^2$.
Though minimizing MSE optimizes fidelity, it does not guarantee that the quantized version is efficiently compressed. We hence deviate from the traditional approach of minimizing MSE and further constrain the compression rate of the quantized gradients: let $\ell_l$ denote the number of bits encoding $s_l$. As we use an entropy encoding, $\ell_l$ only depends on the probability of $s_l$ and not $s_l$ itself. %using Huffman coding. Then, 
The average codeword length after encoding $\mathsf{Q}(z)$ is then given by 
\begin{align}
\mathsf{R}_{\mathsf{Q}} (Z) = \sum_{l \in [2^b]} \ell_l \int_{u_l}^{u_{l+1}} f_Z({z}) \dif {z}.
\end{align}
The rate-constrained quantization is then formulated as
% Hence, our goal is to quantize $z$ to $\mathsf{Q}(z)$ by solving the following optimization problem
\begin{align} \label{eq:RD}
&\min_{\{s_l,u_l: {l \in [2^b]} \} }  \mathsf{MSE}_{\mathsf{Q}} (Z) \;\;
\text{subject to} \;\; \mathsf{R}_{\mathsf{Q}} (Z) \leq R^{\rm trg},
\end{align}
for some pre-defined threshold $R^{\rm trg} > 0 $. %
%To solve the constrained optimization problem in \eqref{eq:RD}, we use 
%Using the Lagrange multipliers technique, \eqref{eq:RD} can be rewritten as unconstrained optimization problem
Alternatively, one can use the Lagrange multipliers technique to write the design in \eqref{eq:RD} in form of a \textit{regularized} distortion optimization as
\begin{align} \label{eq:uncons1}
\min_{\{s_l,u_l: {l \in [2^b]} \} }  \mathsf{MSE}_{\mathsf{Q}} (Z) + \lambda \mathsf{R}_{\mathsf{Q}} (Z),
\end{align}
for a regularizer $\lambda > 0$ that controls distortion-rate trade-off. Using the definitions of MSE and rate, we can simplify \eqref{eq:uncons1} as
\begin{align} \label{eq:uncons}
\min_{\{s_l,u_l: {l \in [2^b]} \} }  \sum_{l \in [2^b]} \int_{u_l}^{u_{l+1}}  \left[ (z-s_l)^2+ \lambda\ell_l \right] f_Z({z}) \dif {z}.
\end{align}

\paragraph{Iterative Optimization}
The optimization in \eqref{eq:uncons} does not have a closed-form solution. We thus invoke the alternating optimization technique to approximate the optimal levels and boundaries iteratively: we \textit{marginally} optimize the levels and boundaries while treating the other as fixed. We then iterate between these two marginal problems till convergence:
\begin{enumerate}
    \item To optimize $s_l$, we note that $\mathsf{R}_{\mathsf{Q}} (Z)$ does not depend on levels $s_l$ and rather $\ell_l$. Thus, in marginal optimization against $s_l$, \eqref{eq:uncons} reduces to classic Lloyd quantizer, which finds the optimal $s_l$ for $l \in [2^b]$ as \cite{lloyd1982least}
\begin{align} \label{eq:s_l}
s_l = \left[ \int_{u_l}^{u_{l+1}} f_Z({z}) \dif {z} \right]^{-1} {\int_{u_l}^{u_{l+1}} {z} f_Z({z}) \dif {z}}.   
\end{align}
\item %Optimizing $\{u_l\}_{l \in [2^b]}$: 
To marginally optimize $u_l$, we note that the objective in \eqref{eq:uncons} can be seen as the expectation of a piece-wise function of $Z$. With optimal boundaries, this function should be \textit{continuous}, i.e., at $z = u_l$ the integrand determined by the piece-function on $(u_{l-1}, u_l]$ should be the same as the one given on $(u_{l}, u_{l+1}]$. This means that % with mapping $z$ to $s_{l-1}$ should be equal to that associated with mapping $z$ to $s_{l}$. In other words, we should have
\begin{align}
(u_l - s_{l-1})^2 + \lambda \ell_{l-1} = (u_l - s_{l})^2 + \lambda \ell_{l},
%\nonumber \\
% \Leftrightarrow \quad & u_l^2 - 2u_l s_{l-1} + s_{l-1}^2 + \lambda \ell_{l-1} = u_l^2 - 2u_l s_{l} + s_{l}^2 + \lambda \ell_{l} \nonumber \\
% \Leftrightarrow \quad & 2 u_l (s_{l} - s_{l-1}) = s_{l}^2 - s_{l-1}^2 + \lambda (\ell_{l} - \ell_{l-1}) \nonumber \\
% \Leftrightarrow \quad & 2 u_l (s_{l} - s_{l-1}) = (s_{l} + s_{l-1})(s_{l} - s_{l-1}) + \lambda (\ell_{l} - \ell_{l-1}) \nonumber \\ \label{eq:u_l}
% \Leftrightarrow \quad & u_l = \frac{1}{2} (s_{l} + s_{l-1}) + \frac{\lambda}{2} \left(\frac{\ell_{l} - \ell_{l-1}}{s_{l} - s_{l-1}}\right).
\end{align}
which after simplification concludes %optimal boundaries as
\begin{align}
% & (u_l - s_{l-1})^2 + \lambda \ell_{l-1} = (u_l - s_{l})^2 + \lambda \ell_{l}  
% %\nonumber \\
% \Leftrightarrow \quad & u_l^2 - 2u_l s_{l-1} + s_{l-1}^2 + \lambda \ell_{l-1} = u_l^2 - 2u_l s_{l} + s_{l}^2 + \lambda \ell_{l} \nonumber \\
% \Leftrightarrow \quad & 2 u_l (s_{l} - s_{l-1}) = s_{l}^2 - s_{l-1}^2 + \lambda (\ell_{l} - \ell_{l-1}) \nonumber \\
% \Leftrightarrow \quad & 2 u_l (s_{l} - s_{l-1}) = (s_{l} + s_{l-1})(s_{l} - s_{l-1}) + \lambda (\ell_{l} - \ell_{l-1}) \nonumber \\ \label{eq:u_l}
u_l = \frac{s_{l} + s_{l-1}}{2} + \frac{\lambda}{2} \left(\frac{\ell_{l} - \ell_{l-1}}{s_{l} - s_{l-1}}\right).\label{eq:u_l}
\end{align}
\end{enumerate} 
The optimal levels and boundaries are computed by iterating between \eqref{eq:s_l} and \eqref{eq:u_l} until a convergence criterion is met. We denote this quantizer by $\mathsf{Q}^{\star}(\cdot)$. Client $k$ then quantizes its normalized gradient as $\hat{\boldsymbol{\mathfrak{g}}}_{k,t}[i] = \mathsf{Q}^{\star}(\Tilde{\boldsymbol{\mathfrak{g}}}_{k,t}[i])$ for $i \in [d]$. %We next discuss the transmission of quantized gradients with the PS.

%$\mathsf{Q}^{\star}(\cdot)$ to quantize all the entries of normalized gradient vector $\Tilde{\boldsymbol{\mathfrak{g}}}_{k,t}$. Let $\hat{\boldsymbol{\mathfrak{g}}}_{k,t}$ denote the quantized version of $\Tilde{\boldsymbol{\mathfrak{g}}}_{k,t}$ obtained using $\mathsf{Q}^{\star}(\cdot)$, i.e., $\hat{\boldsymbol{\mathfrak{g}}}_{k,t}[i] = \mathsf{Q}^{\star}(\Tilde{\boldsymbol{\mathfrak{g}}}_{k,t}[i])$, $\forall i \in [d]$. In the next subsection, we discuss how client $k$ sends the quantized vector $\hat{\boldsymbol{\mathfrak{g}}}_{k,t}$ to the PS.  

\paragraph{Rate-constrained vs Unconstrained}
Comparing with traditional unconstrained quantization \cite{chen2024communication}, the proposed rate-constrained scheme shifts the boundaries to guarantee restricted bit rate: without rate constraint, the boundaries are computed from Lloyd solution as \(u_l = (s_{l} + s_{l-1})/2\). With rate constraint, however, \(u_l\) is shifted towards the reconstruction level associated with the longer codeword; see \eqref{eq:u_l}. Thus, intervals associated with longer codewords become smaller, and longer codewords are chosen with less frequently.

\subsection{Gradient Transmission} \label{sec:trans}
%Due to the rate constraint in \eqref{eq:RD}, the quantized gradients 
Client $k$ transmits $\hat{\boldsymbol{\mathfrak{g}}}_{k,t}$ after encoding it by the entropy coding scheme, e.g., Huffman coding, whose decoder is known to the PS. % to encode the quantization indices significantly reduces the tranmission rate. We denote the Huffman 
We denote the encoder by $\mathsf{enc}(\cdot)$: at iteration $t$, client $k$ sends $\bold{m}_{k,t} =\mathsf{enc}(\hat{\boldsymbol{\mathfrak{g}}}_{k,t}[i])$ for $i \in [d]$ to the PS. In addition, to keep the PS capable of reconstructing the (non-normalized) local gradients, client $k$ transmits mean $\mu_{k,t}$ and standard deviation $\sigma_{k,t}$. For this transmission, it uses full-precision, i.e., 32 bits, requiring a total of 64 extra bit transmissions. % to complete this transmission. %t $\mu_{k,t}$ and $\sigma_{k,t}$. 

% This amount is negligible compared to the number of bits required for transmitting the quantized gradient vectors. 

\subsection{Gradient Accumulation} \label{sec:PS}
Upon receiving $\bold{m}_{k,t}$, the PS uses decoder $\mathsf{dec}(\cdot)$ to retrieve $\hat{\boldsymbol{\mathfrak{g}}}_{k,t} = \mathsf{dec}(\bold{m}_{k,t})$. It then uses the inverse function of $\mathsf{Q}^{\star}(\cdot)$, denoted by $\mathsf{Q}^{\star}_{\rm i}(\cdot)$, along with $\mu_{k,t}$ and $\sigma_{k,t}$ to reconstruct %local gradient entries as
\begin{align} \label{eq:reconst}
\breve{\boldsymbol{\mathfrak{g}}}_{k,t}[i]  =  \sigma_{k,t} \mathsf{Q}^{\star}_{\rm i} (\hat{\boldsymbol{\mathfrak{g}}}_{k,t}[i])  + \mu_{k,t}.
\end{align}
Finally, it computes
$\bar{\boldsymbol{\mathfrak{g}}}_{t}$ by averaging %\frac{1}{K} \sum\nolimits_{k \in [K]} 
$\breve{\boldsymbol{\mathfrak{g}}}_{k,t}$ over the clients and updates the global model as
$\boldsymbol{\theta}_{t+1}=\boldsymbol{\theta}_{t}-\eta_t \bar{\boldsymbol{\mathfrak{g}}}_{t}$.    
The proposed scheme is summarized in Algorithm \ref{alg}.

\begin{algorithm}[t]
\caption{RC-FED} \label{alg}
\small
{\bfseries Input:} Number of rounds ${T}^{\max}$, learning rate ${\eta}_t$, initial global model $\boldsymbol{\theta}_{0}$, local datasets, and quantizer $\mathsf{Q}^{\star}(\cdot)$\\ %, regularizer $\lambda$  \\
\For{$t=0,1,\dots,\text{T}^{\rm max}-1$}
{
    PS sends $\boldsymbol{\theta}_{t}$ to all the clients\\
    \For{client $k \in [K]$ in parallel}{
        Compute $ \boldsymbol{\mathfrak{g}}_{k,t}$ %$= \frac{1}{\left|\mathcal{D}_{i}\right|} \sum_{\bold{x} \in \mathcal{D}_{i}} \nabla f_k \left(\boldsymbol{\theta}_{t} ,\bold{x} \right)$ using SGD 
        over a randomly chosen mini-batche\\
        Compute $\mu_{k,t}$ and $\sigma_{k,t}$ and normalize $ \boldsymbol{\mathfrak{g}}_{k,t}$ to $\Tilde{\boldsymbol{\mathfrak{g}}}_{k,t}$ \\
        Quantize $\hat{\boldsymbol{\mathfrak{g}}}_{k,t}[i] = \mathsf{Q}^{\star}(\Tilde{\boldsymbol{\mathfrak{g}}}_{k,t}[i])$ for $i \in [d]$\\
        Encode quantized gradient as $\bold{m}_{k,t} = \mathsf{enc}(\hat{\boldsymbol{\mathfrak{g}}}_{k,t})$ \\
        Send $\bold{m}_{k,t}$ and $(\mu_{k,t},\sigma_{k,t})$ to the PS
    }

    PS computes $\breve{\boldsymbol{\mathfrak{g}}}_{k,t}$ as per \eqref{eq:reconst}, and finds $\bar{\boldsymbol{\mathfrak{g}}}_{t}$\\
    PS updates the global model as $\boldsymbol{\theta}_{t+1}=\boldsymbol{\theta}_{t}-\eta_t \bar{\boldsymbol{\mathfrak{g}}}_{t}$
}
\textbf{Output:} Global model $\boldsymbol{\theta}_{{T}^{\max}}$ 
\end{algorithm}

\begin{figure*}[!t] 
\vskip -0.2in
  \centering 
  \subfloat[CIFAR-10.]{\includegraphics[width=0.45\linewidth]{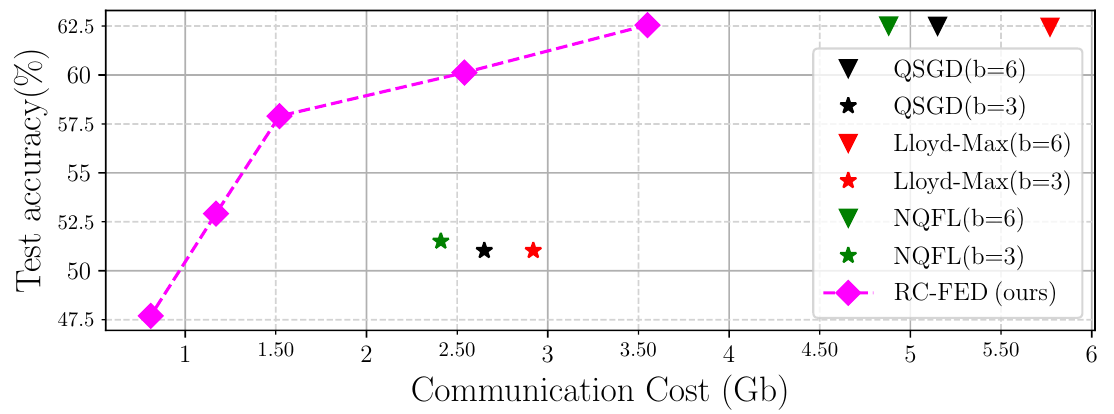}
\label{fig:cifar10}}
  \subfloat[FEMNIST.]{\includegraphics[width=0.45\linewidth]{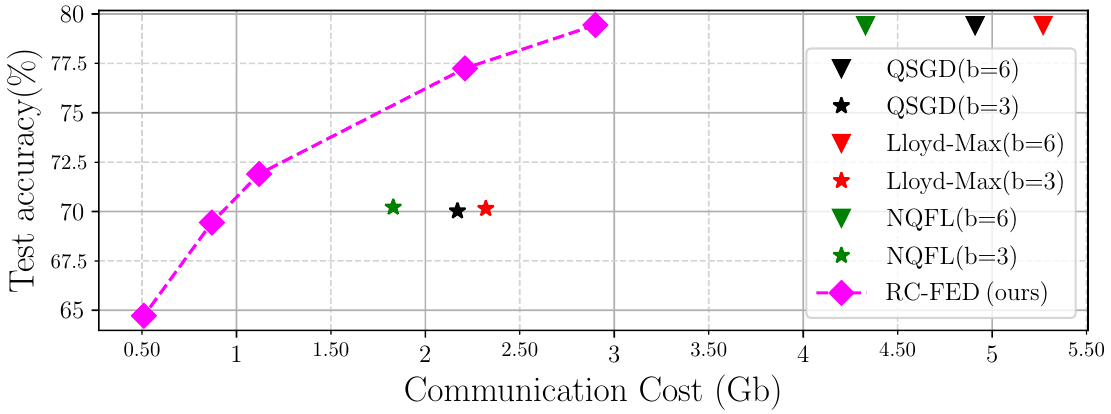}
\label{fig:fmnist}}
  \caption{Test accuracy vs communication costs in Gb for RC-FED and baselines over (a) CIFAR-10, and (b) FEMNIST.}  \vskip -0.1in \label{fig:ImageNet}

\end{figure*}

\section{Convergence Analysis}
We nest study the convergence properties of RC-FED. We start the analysis by stating the set of analytic assumptions on local gradients and losses \cite{li2019convergence,9305988}: we assume that
\begin{itemize}
    \item[(A-I)] For all $\boldsymbol{\theta} \in \mathbb{R}^d$, the expected squared Euclidean distance of local gradient $\boldsymbol{\mathfrak{g}}_{k,t}$ is bounded by some $\zeta^2_k$.
    \item[(A-II)] For ${k \in [K]}$, the second moment of the local gradient %for all
%local loss functions $\{ f_k(\cdot)\}_{k \in [K]}$ 
is bounded by $\xi^2$ for all $\boldsymbol{\theta} \in \mathbb{R}^d$.
\item[(A-III)] For ${k \in [K]}$, the local loss function $f_k(\cdot)$ is $L$-smooth, i.e., for any $\boldsymbol{v}_1, \boldsymbol{v}_2 \in \mathbb{R}^d$, we have $f_k(\boldsymbol{v}_1) - f_k(\boldsymbol{v}_2) \leq (\boldsymbol{v}_1 - \boldsymbol{v}_2)^{\sf T} \nabla f_k(\boldsymbol{v}_2) + \frac{L}{2} \| \boldsymbol{v}_1 - \boldsymbol{v}_2 \|^2$. 
\item[(A-IV)] For ${k \in [K]}$, the local loss function $f_k(\cdot)$ is $\rho$-strongly
convex, i.e., for any $\boldsymbol{v}_1, \boldsymbol{v}_2 \in \mathbb{R}^d$, we have $f_k(\boldsymbol{v}_1) - f_k(\boldsymbol{v}_2) \geq (\boldsymbol{v}_1 - \boldsymbol{v}_2)^{\sf T} \nabla f_k(\boldsymbol{v}_2) + \frac{\rho}{2} \| \boldsymbol{v}_1 - \boldsymbol{v}_2 \|^2$.
\end{itemize}
% \vskip -0.2in

Under these assumptions, we characterize the optimality gap of RC-FED. To this end, let  $\boldsymbol{\theta}^{\star}$ denote the solution of minimization \eqref{eq:FL}, and define the \textit{heterogeneity gap} as \cite{9305988}
\[
\Gamma \defeq f(\boldsymbol{\theta}^{\star}) - \frac{1}{K} \sum_{k \in [K]} \min_{\boldsymbol{\theta}} f_k(\boldsymbol{\theta}).
\]
The following theorem establishes an upper bound on the optimality gap at each iteration of RC-FED with an arbitrary number of local iterations at each client. %
%the convergence behaviour of RC-FED whose proof is differed to the \textit{supplementary materials}:
\begin{theorem} \label{th:conv}
Let assumptions (A-I) to (A-IV) hold. Assume that all clients perform $e$ local iterations, and that $\eta_t = \frac{2}{\rho (t+\gamma)}$ for $\gamma = \max \{ 8{L}/{\rho} , e \}-1$. Define the optimality gap at round $t$ as 
$\Delta_t= \mathbb{E} \{ f(\boldsymbol{\theta}_{t}) - f(\boldsymbol{\theta}^{\star})\}$.
Then, starting with $\boldsymbol{\theta}_0$, we have
\begin{align}
\Delta_t \leq \frac{L}{2(t+\gamma)} \max  \Big\{ \frac{4C}{\rho^2} , (\gamma+1) \mathbb{E} \big\{ \| \boldsymbol{\theta}_{0} - \boldsymbol{\theta}^{\star}\|^2 \big\}\Big\},
\end{align}
where $C$ is given by
\[
C =  \frac{\pi e}{6K} \sum_{k \in [K]} \sigma^2_{k,t} 2^{-2\mathsf{R}_{\sf Q^{\star}} (Z)} +6L\Gamma +\frac{8(e-1)}{K} \sum_{k \in [K]} \zeta^2_k.
\]
\end{theorem}

 \begin{proof}
 Please refer to Appendix \ref{sec:proof}.
     % The proof is skipped here, due to lack of space and given in the appendix of extended version of the work.
 \end{proof}

Theorem \ref{th:conv} implies that the convergence rate of RC-FED is $\mathcal{O}({1}/{t})$, which is align with that of \cite{9305988}. This means that the rate constraint we imposed in RC-FED does not slow down the convergence of FL algorithm. 

\vskip -0.1in

\section{Numerical Results}
We evaluate RC-FED through some numerical experiments and compare it against the state-of-the-art. As baseline % to demonstrate and evaluate RC-FED, comparing it with some conterparts.
%
%\noindent $\bullet$ \textbf{Benchmarks:} 
we consider {traditional QSGD} \cite{alistarh2017qsgd}, {Lloyd-Max quantizer} \cite{chen2024communication}, and NQFL \cite{chen2023nqfl}. We test the benchmark schemes for $b=\{3,6\}$. For a fair comparison, we use Huffman coding to compress the quantized gradients in all methods.
% In addition, we compare RC-FED with the benchmark methods when integrated with a adaptive quantization method, namely  AdaQ \cite{jhunjhunwala2021adaptive}, that dynamically adjust the number of quantized levels according to the loss function. For a fair comparison, we use Huffman coding in the benchmark methods to encode their respective gradient vectors. 
%\noindent $\bullet$ \textbf{Performance metrics:} In all the experiments, 
We compute the test accuracy of the trained models and sketch them against the communication costs required for the entire training.

%\noindent $\bullet$ \textbf{Datasets:} We perform 
We first consider CIFAR-10 with similar setups to those in \cite{hamidi2024adafed,10381881}.
%\noindent $\bullet$ \textbf{CIFAR-10 setup:}
% \subsection{CIFAR-10}
%This dataset contains 50K training and 10K test images labeled for 10 classes. 
We distribute the dataset among $K=10$ clients using Dirichlet allocation with $\beta=0.5$ and train \textit{ResNet-18} for 100 communication rounds with single local iterations. The batch size is set to 64 and $\eta_t=0.01$.

%\noindent $\bullet$ \textbf{FEMNIST setup:}
We next consider the federated extended MNIST (FEMNIST) dataset, which is a federated image classification dataset distributed over 3550 devices with 62 classes \cite{caldas2018leaf}. We train a classic {CNN} with two convolutional layers followed by two fully-connected layers. The batch size is 32. At each round, $K=500$ devices are randomly sampled out of the 3550 ones and locally train for $e=2$ local iterations using the default data stored in each device. 

\paragraph{Performance Comparison} The results for CIFAR-10 and FEMNIST datasets are depicted in Figs. \ref{fig:cifar10} and \ref{fig:fmnist}, respectively. % the number of bits used in the benchmarks for gradient quantization is written in front of the respective methods. Moreover, for 
For RC-FED, we plot the results for various $0.02 \leq \lambda \leq 0.1$ resulting in a \textit{curve}. As observed, for the same test accuracy, RC-FED requires significantly lower communication costs. For instance, considering CIFAR-10 dataset, RC-FED achieves $62.52\%$ and $52.92\%$ test accuracy with 3.55 Gigabits (Gb) and 1.17 Gb of data transmission, which is significantly lower than that required by the benchmarks. This empirically validate the effectiveness of quantization with rate constraint.

% Note that the quantized gradients are encoded using Huffman coding for both benchmark methods and RC-FED; thus, the effectiveness of RC-FED lies in the fact that the quantized gradients obtained within the RC-FED framework are compressible in an information-theoretic sense.

\vskip -0.2in

\section{Conclusion}
This paper introduced RC-FED, a quantized FL framework in which clients quantize their local gradients by simultaneously minimizing distortion and restricting the data rate achieved after encoding the quantized gradients. Numerical results showed that RC-FED ($i$) significantly reduces the communication overhead as compared with baselines, while maintaining the test accuracy, and ($ii$) establishes a tunable trade-off between communication overhead and test accuracy. A natural direction for future work is to extend the RC-FED framework beyond scalar quantization. 

% Through analytical analysis, we demonstrated the convergence behavior of RC-FED, and experimental evaluations on two datasets validated its effectiveness. 

\bibliography{main}

% \begin{thebibliography}{1}
\bibliographystyle{IEEEtran}
% \end{thebibliography}{1}

\clearpage
\newpage

\section{Proof of Theorem 1} \label{sec:proof}
If \textbf{A(III)} and \textbf{A(IV)} hold, and $\eta_t \leq \frac{1}{4L}$, then based on \cite{li2019convergence} (see Lemma 1 therein), we have 
\small
\begin{align}
& \mathbb{E} \{ \| \boldsymbol{\theta}_{t+1} - \boldsymbol{\theta}^{\star} \|^2 \}  \leq (1 - \eta_t \rho) \mathbb{E} \{ \| \boldsymbol{\theta}_{t} - \boldsymbol{\theta}^{\star} \|^2 \}  + 6L \eta_t^2 \Gamma \nonumber \\ \label{eq:bound}
& \quad + 2  \mathbb{E} \Big\{ \frac{1}{K} \sum_{k \in [K]} \| \boldsymbol{\theta}^t - \boldsymbol{\theta}_{k,t}\|^2\Big\} + \eta_t^2 \mathbb{E} \Big\{ \| \bar{\boldsymbol{\mathfrak{g}}}_{t} - \frac{1}{K} \sum_{k \in [K]} \boldsymbol{\mathfrak{g}}_{k,t} \|^2 \Big\} ,
\end{align}
\normalsize
where $\boldsymbol{\theta}_{k,t} = \boldsymbol{\theta}_{k,t-1} - \eta_{t-1}\mathfrak{g}_{k,t}$ is the updated model parameter at client $k$. To bound $\mathbb{E} \Big\{ \frac{1}{K} \sum_{k \in [K]} \| \boldsymbol{\theta}^t - \boldsymbol{\theta}^t_k\|^2\Big\}$ and $\mathbb{E} \Big\{ \| \bar{\boldsymbol{\mathfrak{g}}}_{t} - \frac{1}{K} \sum_{k \in [K]} \boldsymbol{\mathfrak{g}}_{k,t} \|^2 \Big\}$ in \eqref{eq:bound}, we introduce Lemmas \ref{lem:1} and \ref{lem:2} whose proofs are provided in \textit{supporting documents}:
\begin{lemma} \label{lem:1}
Assume that assumption \textbf{A(I)} holds, and that local clients perform $e$ local epochs of training, then
\begin{align}
\mathbb{E} \Big\{ \frac{1}{K} \sum_{k \in [K]} \| \boldsymbol{\theta}^t - \boldsymbol{\theta}^t_k\|^2\Big\} \leq 4\eta_t^2 (e-1)^2 \frac{1}{K} \sum_{k \in [K]}  \zeta^2_k.    
\end{align}    
\end{lemma}

\begin{lemma} \label{lem:2}
Assume that the gradient distribution for client $k$ follows Gaussian distribution with standard deviation $\sigma_{k,t}$, and that the clients use $\mathsf{Q}^{\star}(\cdot)$ for gradient quantization, then 
\small
\begin{align}
\mathbb{E} \Big\{ \| \bar{\boldsymbol{\mathfrak{g}}}_{t} - \frac{1}{K} \sum_{k \in [K]} \boldsymbol{\mathfrak{g}}_{k,t} \|^2 \Big\}  \leq \frac{\pi e}{6K}  \sum_{k \in [K]} \sigma^2_{k,t} 2^{-2\mathsf{R}_{\sf Q} (Z)}. 
\end{align}   
\normalsize
\end{lemma}

Using Lemmas 1 and 2 in \eqref{eq:bound}, we obtain
\small
\begin{align} \label{eq:recurs}
\mathbb{E} \{ \| \boldsymbol{\theta}_{t+1} - \boldsymbol{\theta}^{\star} \|^2 \}  \leq (1 - \eta_t \rho) \mathbb{E} \{ \| \boldsymbol{\theta}_{t} - \boldsymbol{\theta}^{\star} \|^2 \} + \eta_t^2 C. 
\end{align}
\normalsize
which is a recursive expression. The theorem can be concluded by using \textbf{A(III)} and induction over $t$ in \eqref{eq:recurs} (the full proof is analogous to the proof of Theorem 1 in \cite{li2019convergence}).

\subsection{Proof of Lemma 1} \label{app:prooflem1}

% First, we define $\bar{\boldsymbol{\theta}}_t = \frac{1}{K} \sum_{k \in [K]} \boldsymbol{\theta}_{k,t}$

Since the clients perform $e$ local epochs, there exists $t_0 \leq t$ such that $t - t_0 \leq e-1$, and $\boldsymbol{\theta}_{k,t_0} = \boldsymbol{\theta}_{t_0}$, $\forall k \in [K]$. Also, note that $\eta_t$ is non-increasing, and $\eta_{t_0} \leq 2 \eta_t$. Then, we have
\small
\begin{align}
& \mathbb{E} \Big\{ \frac{1}{K} \sum_{k \in [K]} \| \boldsymbol{\theta}^t - \boldsymbol{\theta}^t_k\|^2\Big\} \nonumber \\
&= \frac{1}{K} \mathbb{E} \Big\{ \sum_{k \in [K]} \| (\boldsymbol{\theta}^t_k - \boldsymbol{\theta}_{t_0}) - (\boldsymbol{\theta}^t-\boldsymbol{\theta}_{t_0})\|^2\Big\} \nonumber \\
& \leq \frac{1}{K} \mathbb{E} \Big\{ \sum_{k \in [K]} \| \boldsymbol{\theta}^t_k - \boldsymbol{\theta}_{t_0}\|^2\Big\} \leq  \sum_{k \in [K]} \frac{1}{K} \sum_{t=t_0}^{t-1}(e-1) \eta_{t_0}^2 \| \boldsymbol{\mathfrak{g}}_{k,t}\|^2 \nonumber \\
& \leq \sum_{k \in [K]} \frac{1}{K} \eta_{t_0}^2 (e-1)^2 \| \boldsymbol{\mathfrak{g}}_{k,t}\|^2 \leq 4\eta_t^2 (e-1)^2 \frac{1}{K} \sum_{k \in [K]}  \zeta^2_k,
\end{align}
\normalsize
which concludes Lemma \ref{lem:1}.

\subsection{Proof of Lemma 2} \label{app:prooflem2}
Here, entropy and differential entropy are denoted by $\mathsf{H}(\cdot)$ and $\mathsf{h}(\cdot)$, respectively.

First note that using  Lloyd quantizers, the MSE distortion becomes  $\mathsf{MSE} (z,\mathsf{Q}(z)) = \frac{1}{12} \sum_{l \in [2^b]} p_l \Delta_l^2$ \cite{panter1951quantization}.

Assume a high-rate scenario, where $b$ is large enough. Then, the PDF $ f_Z(z)$ is almost constant in each quantization interval. As such, $f_Z(s_l) \approx \frac{p_l}{\Delta_l} = \frac{p_l}{u_{l+1}-u_{l}}$, and thus $p_l \approx f_Z(s_l)\Delta_l$. Hence, the rate could be obtained as follows
\small
\begin{align}
& \mathsf{R}_{\sf Q^{\star}} (Z) = \mathsf{H} (\mathsf{Q}^{\star}(z)) = - \sum_{l \in [2^b]} p_l \log p_l \nonumber \\
& = - \sum_{l \in [2^b]} p_l (\log f_Z(s_l) + \log \Delta_l)\nonumber \\
& = - \sum_{l \in [2^b]} f_Z(s_l) \log f_Z(s_l) \Delta_l - \sum_{l \in [2^b]} f_Z(s_l) p_l \log \Delta_l \nonumber \\
& \approx \int_{-\infty}^{\infty} f_Z(\Tilde{s}) \log f_Z(\Tilde{s}) d \Tilde{s} - \frac{1}{2}  \sum_{l \in [2^b]} p_l \log \Delta_l^2 \nonumber \\ \label{eq:h}
& = \mathsf{h}(Z) - \frac{1}{2}  \sum_{l \in [2^b]} p_l \log \Delta_l^2  \geq \mathsf{h}(Z) - \frac{1}{2} \log \left( \sum_{l \in [2^b]} p_l \Delta_l^2 \right) \\ \label{eq:last}
&=  \mathsf{h}(Z) - \frac{1}{2} \log \big( 12~\mathsf{MSE} (z,\mathsf{Q}^{\star}(z)) \big),
\end{align}
\normalsize
where in \eqref{eq:h}, $\mathsf{h}(\cdot)$ denotes the differentiable entropy, and we used Jensen’s inequality to derived the inequality in \eqref{eq:h}. From \eqref{eq:last}, it is concluded that
\begin{align} \label{eq:MSE}
\mathsf{MSE} (z,\mathsf{Q}^{\star}(z)) = \frac{1}{12} 2^{2 \mathsf{h}(Z)}  2^{-2\mathsf{R}_{\sf Q^{\star}} (Z)} . 
\end{align}
For $Z \sim \mathcal{N} (\mu,\sigma^2)$, the expression in \eqref{eq:MSE} is simplified to 
\begin{align}
\frac{\pi e}{6} \sigma^2 2^{-2\mathsf{R}_{\sf Q^{\star}} (Z)}.  
\end{align}
Knowing that $\boldsymbol{\mathfrak{g}}_{k,t} \sim \mathcal{N} (\mu_{k,t}, \sigma^2_{k,t})$, then we have
\begin{align}
\mathbb{E} \Big\{ \| \bar{\boldsymbol{\mathfrak{g}}}_{t} - \frac{1}{K} \sum_{k \in [K]} \boldsymbol{\mathfrak{g}}_{k,t} \|^2 \Big\}  \leq \frac{\pi e}{6K}   \sum_{k \in [K]} \sigma^2_{k,t} 2^{-2\mathsf{R}_{\sf Q^{\star}} (Z)}, 
\end{align}
which concludes the Lemma \ref{lem:2}. 
\end{document}